\documentclass[12pt]{article}

\usepackage{amsmath}
\usepackage{amssymb}
\usepackage{mathtools}
\usepackage{amsthm}
\usepackage{graphicx}
\usepackage{caption}
\usepackage{subcaption}

\usepackage{amsmath,amsthm,amssymb}
\usepackage{xcolor}
\usepackage{algorithm}
\usepackage{algorithmicx}
\usepackage{algpseudocode}
\usepackage{fullpage}

\usepackage{amsmath,amsfonts,amssymb}

\usepackage{setspace}
\usepackage{bm}
\usepackage{bbm}
\usepackage{tikz}
\usepackage{framed}
\usepackage[utf8]{inputenc}
\usepackage[T1]{fontenc}
\usepackage[round]{natbib}

\usepackage{hyperref}
\usepackage[capitalize,noabbrev]{cleveref}
\hypersetup{colorlinks,
            linkcolor=blue,
            citecolor=blue,
            urlcolor=magenta,
            linktocpage,
            plainpages=false}

\theoremstyle{plain}
\newtheorem{theorem}{Theorem}[section]

\newtheorem{lemma}[theorem]{Lemma}

\newtheorem{assumption}[theorem]{Assumption}
\newtheorem{model}[theorem]{Model}
\theoremstyle{remark}
\newtheorem{remark}[theorem]{Remark}
\newtheorem{property}[theorem]{Property}
\newtheorem{example}[theorem]{Example}

\usepackage{authblk}
\newcounter{@inst}
\newcounter{@auth}

\newdimen\instindent
\newbox\authrun
\newtoks\authorrunning
\newtoks\tocauthor
\newbox\titrun
\newtoks\titlerunning
\newtoks\toctitle

\def\clearheadinfo{\gdef\@author{No Author Given}%
                  \gdef\@title{No Title Given}%
                  \gdef\@subtitle{}%
                  \gdef\@institute{No Institute Given}%
                  \gdef\@thanks{}%
                  \global\titlerunning={}\global\authorrunning={}%
                  \global\toctitle={}\global\tocauthor={}}

\def\sw{w^*}
\newcommand\bigO[1]{\cO \left( #1 \right)}
\newcommand\bigtO[1]{\tilde{\cO} \left( #1 \right)}
\newcommand\ofrac[1]{\frac{1}{#1}}
\newcommand\onicefrac[1]{\nicefrac{1}{#1}}
\newcommand\inn[1]{\left\langle #1 \right\rangle}
\newcommand\expc[2]{\underset{#1}{\E} \left[ #2 \right]}
\newcommand\cexpc[3]{\underset{#1}{\E} \left[ #2  \Big| #3 \right]}
\newcommand\hubr[2]{h_{#1} \left( #2 \right)}
\newcommand\hubR[1]{\hubr{R}{#1}}
\newcommand\phiR[1]{\phi_R \left(#1\right)}
\newcommand\abs[1]{\left| #1 \right|}
\newcommand\norm[1]{\left\| #1 \right\|}

\def\bI{\mathbf{I}}

\def\E{\mathbb{E}}
\def\R{\mathbb{R}}
\def\N{\mathbb{N}}

\def\Ex{\expc{}{x}}
\newcommand\prob[1]{\mathbb{P} \left( #1 \right)}
\DeclareMathOperator*{\argmin}{\arg\min}
\def\half{\ofrac{2}}
\def\g{\nabla}

\def\logdelta{\log \onicefrac{\delta}}

\def\sqrtdeltaT{\sqrt{\frac{\log \onicefrac{\delta}}{T}}}

\def\cW{\mathcal{W}}
\def\cB{\mathcal{B}}

\def\cN{\mathcal{N}}
\def\cP{\mathcal{P}}
\def\cQ{\mathcal{Q}}

\def\cO{\mathcal{O}}

\def\cF{\mathcal{F}}
\def\cG{\mathcal{G}}

\def\eps{\epsilon}

\def\wopt{w_{opt}}

\newcommand\lemref[1]{Lemma~\ref{lem:#1}}

\newcommand\thmref[1]{Theorem~\ref{thm:#1}}
\newcommand\mdlref[1]{Model~\ref{mdl:#1}}
\newcommand\propref[1]{Property~\ref{prop:#1}}
\newcommand\algoref[1]{Algorithm~\ref{algo:#1}}

\newcommand\asmpref[1]{Assumption~\ref{asmp:#1}}
\newcommand\sectref[1]{Section~\ref{sect:#1}}

\newcommand\eqnref[1]{Equation~(\ref{eq:#1})}
\newcommand\exmpref[1]{Example~(\ref{exmp:#1})}

\newcommand\sumt[2]{\sum_{t=#1}^{#2}}
\def\sumT{\sumt{1}{T}}
\def\meanT{\ofrac{T} \sumT}

\usepackage{nicefrac}

\newcommand\proj[2]{\Pi_{#1} \left( #2 \right)}

\def\tL{\tilde{L}}

\def\bw{\bar{w}}

\title{Robust Linear Regression for General Feature Distribution}
\author[1]{Tom Norman}
\author[1]{Nir Weinberger}
\author[1,2]{Kfir Y. Levy}
\affil[1]{Department of Electrical and Computer Engineering, Technion}
\affil[2]{A Viterbi Fellow}
\date{}

\begin{document}
\maketitle

\begin{abstract}
    We investigate robust linear regression where data may be contaminated by an oblivious adversary, i.e., an adversary than may know the data distribution but is otherwise oblivious to the realizations of the data samples. This model has been previously analyzed under strong assumptions. Concretely, \textbf{(i)} all previous works assume that the covariance matrix of the features is positive definite; and \textbf{(ii)} most of them assume that the features are centered (i.e. zero mean). Additionally, all previous works make additional restrictive assumption, e.g., assuming that the features are Gaussian or that the corruptions are symmetrically distributed.
    
    In this work we go beyond these assumptions and investigate robust regression under a more general set of assumptions: \textbf{(i)} we allow the covariance matrix to be either positive definite or positive semi definite, \textbf{(ii)} we do not necessarily assume that the features are centered, \textbf{(iii)} we make no further assumption beyond boundedness (sub-Gaussianity) of features and measurement noise.
    Under these assumption we analyze a natural SGD variant for this problem and show that it enjoys a fast convergence rate when the covariance matrix is positive definite. In the positive semi definite case we show that there are two regimes: if the features are centered we can obtain a standard convergence rate; otherwise the adversary can cause any learner to fail arbitrarily.
\end{abstract}

\section{Introduction}\label{sect:introduction}

The remarkable recent success of Machine Learning (ML) models has lead to their wide adoption in numerous fields. 
However, deploying ML models in real world scenarios brings several challenges. One such challenge of paramount importance is 
\emph{robustness}, i.e., the need to design models that are immune to data contamination. The latter might arise due to
adversarial corruptions, extreme events, or malfunctioning sensors, amongst other causes. Even beyond ML, designing robust models has proven to be crucial in various applications, including  Economics~\citep{zaman2001econometric}, Computer Vision~\citep{gustafsson2020evaluating}, Biology~\citep{yeung2002reverse}, and Healthcare~\citep{davies2004robust}.
 
In this paper, we investigate robustness in the context \emph{linear regression}, one of the most fundamental ML tasks.
Concretely, we explore robust regression under the assumption that a fraction $\alpha$ of the observations were \emph{contaminated by an adversary}.
In this context, it is well known that standard regression methods are highly sensitive to outliers, and might break down even in the presence of a single contaminated data point. 

Past research on robust linear regression roughly falls into one of two categories, depending on the power of the adversary:
\textbf{(a)} An \emph{adaptive adversary} is allowed to contaminate the data \textit{after} observing the data samples.
This setting was explored, e.g., in~\citet{candes2005decoding,charikar2017learning,klivans2018efficient,liu2019high,dalalyan2019outlier,diakonikolas2019robust,Diakonikolas2019EfficientAA}.
It is well known that {adaptive adversaries} may cause any learner to incur a non vanishing error, depending on the fraction of contaminated samples, irrespective of the number of data points. Conversely, \textbf{(b)} an \emph{oblivious adversary} is not allowed to observe the samples, but may know the true statistical properties of the data. This setting was explored, e.g., in \citet{ts2014,suggala2019adaptive,NEURIPS2020_1ae6464c,sun2020adaptive,pmlr-v139-d-orsi21a}, which have remarkably shown that one can obtain vanishing error by increasing the number of data samples, for \textit{any} fraction of contamination.

In this paper we focus on robust linear regression with oblivious adversaries. While past works on this topic have highly advanced the understanding of this setting, they were done under limiting assumptions. 
Concretely: \textbf{(i)} all previous works assume that the covariance matrix of the features $\Sigma$ is positive definite; and \textbf{(ii)} most works assume that the features are centered (i.e., have zero mean). Additionally, all previous works make additional restrictive assumptions, e.g., that the features are Gaussian or that the corruptions are symmetrically distributed.

In this work, we analyze robust regression under a broader set of assumptions.
Concretely, our work applies to general feature and noise distributions with bounded (or just sub-Gaussian) distributions. Our contributions are as follows (see Table \ref{sample-table}):
\begin{itemize}
\item For the case of $\Sigma\succeq 0$ we analyze two regimes: Under the assumption that the features are \emph{centered}, we provide an SGD (Stochastic Gradient Descent) variant that ensure an error rate of $\bigO{\onicefrac{(1-\alpha) \sqrt{T}}}$ after observing $T$ samples. This is the first result for 
this case. Conversely, if the features are non-centered, we show that any algorithm may completely fail. 
\item For strictly positive definite $\Sigma$, we provide an SGD variant that ensures an estimation error of $\bigtO{\onicefrac{(1-\alpha)^2 T}}$. 
This is done without any assumption on the centering of the features.
\end{itemize}
We make no further assumptions regarding the adversary/data. Moreover, we provide SGD variants that do not require the knowledge of the contamination fraction $\alpha$. Finally, we allow the adversary to inject unbounded perturbations into the contaminated measurements.

On the technical side, our work builds on utilizing the Huber loss~\citep{huber1964}, which is a robust loss function, instead of the $\ell_2$ loss.
This is done in conjunction to SGD variants that employ feature centering.

\subsection*{Related Work}
Robust statistics dates back to the works of Tukey and Huber \citep{tukey1960survey,huber1964},
Classical robust statistics has mainly focused on asymptotic performance \citep{huber1973robust,bassett1978asymptotic,pollard1991asymptotics,van2000asymptotic,mcmahan2013ad}, and many of the approaches were not computable in polynomial time \citep{rousseeuw1984least,rousseeuw1985multivariate}.

A popular approach towards robust regression with oblivious adversaries relies on replacing the $\ell_2$ loss with more robust loss, predominantly either
the $\ell_1$ loss or the Huber loss~\citep{huber1964} (which are convex), as well other non-convex robust losses \citep{tukey1960survey}.

\paragraph{Finite Time Guarantees for Robust regression:} Non-asymptotic guarantees for robust regression were recently explored in several works; all of them rely on employing a convex robust loss function (either $\ell_1$ or Huber). Moreover, as we detail in Table \ref{sample-table}, all previous works assume strictly positive definite covariance matrix, i.e. $\Sigma\succeq \rho \cdot I$ for $\rho > 0$, and centered features, in addition to other restrictive assumptions detailed below (see also Table \ref{sample-table}).

\citet{ts2014}, assume that the features $x$ and measurement noise $\epsilon$ have zero-mean Gaussian distribution. 
Their algorithmic approach is to apply ERM (Empirical Risk Minimization) while utilizing the Huber loss.
\citet{suggala2019adaptive} similarly assume that $x$ is Gaussian yet allow the measurement noise be sub-Gaussian. They suggest an algorithm, AdaCRR, which makes several passes over the dataset while thresholding suspicious points. \citet{NEURIPS2020_1ae6464c} makes the same Gaussianity assumptions as \citet{ts2014}, and are the first to provide guarantees for an efficient online algorithm, namely SGD with $\ell_1$ loss.

The recent work of \citet{pmlr-v139-d-orsi21a} has significantly improved the theoretical understanding by relaxing the Gaussian assumptions of previous works. Similarly to \citet{ts2014} they provide guarantees to ERM over the Huber loss. Nevertheless, they make two limiting assumptions.
First, they assume that both the measurement noise and adversarial perturbations are \emph{symmetrically distributed around zero}, which highly weakens the adversary. They also make an assumption called \emph{spreadness} regarding the features, which limits the setup (Nonetheless, they allow the features to be non-centered). 

\begin{table}[t!]
\Large
\centering
\resizebox{\textwidth}{!}{\begin{tabular}{||lcccc||}
 \hline
 Paper & Features & Noise \& Adversary & Rates for $\Sigma \succ 0$ & Rates for $\Sigma \succeq 0$\\ [1.3ex] 
 \hline\hline
 \citet{ts2014} & $x \sim \cN(0, I_d)$ & $\epsilon \sim \cN(0, \sigma^2)$ & $\bigtO{\onicefrac{(1-\alpha)^2 T}}$ & N \slash A\\
 \hline
 \citet{suggala2019adaptive} & $x \sim \cN(0, \Sigma)$ & $\epsilon \sim subG(\sigma^2)$ & $\bigtO{\onicefrac{(1-\alpha)^2 T}}$ & N \slash A\\
 \hline
 \citet{NEURIPS2020_1ae6464c} & $x \sim \cN(0,\Sigma)$ & $\epsilon \sim \cN(0, \sigma^2)$ & $\bigtO{\onicefrac{(1-\alpha)^2 T}}$ & N \slash A\\
 \hline
 \citet{pmlr-v139-d-orsi21a} & \begin{tabular}[c]{@{}c@{}}A deterministic\\design matrix $X$\end{tabular} & \begin{tabular}[c]{@{}c@{}}$\epsilon + b$ has a symmetric\\distribution around 0\end{tabular} & $\bigtO{\onicefrac{(1-\alpha)^2 T}}$ & N \slash A\\
 \hline
 This paper & $x \sim subG(\kappa)$ & $\epsilon \sim subG(\sigma^2)$ & $\bigtO{\onicefrac{(1-\alpha)^2 T}}$ & $\bigO{\onicefrac{(1-\alpha) \sqrt{T}}}$\\ [1.6ex] 
 \hline
\end{tabular}}
\caption{Assumptions and rates of related work.}
\label{sample-table}
\end{table}
\section{Problem Formulation}\label{sect:formulation}

We consider a linear model where the observations may be contaminated by an oblivious adversary,
\begin{model}\label{mdl:simple_model}
\qquad \qquad \qquad 
    $\qquad y = \inn{\sw, x} + \epsilon + b$~,
\end{model}
where $x\in\R^d$ is a feature vector, $\epsilon\in\R$ is a zero-mean additive noise that is statistically independent of $x$, and $\langle\cdot,\cdot\rangle$ denotes the standard inner product. The corruptions $b\in\R$ is chosen by an adversary that  knows $\sw$ as well as the probability distributions of $x$ and $\epsilon$, but is otherwise \emph{oblivious to their realizations}. We assume that adversary chooses the corruptions $b$ according to some probability distribution that is constrained to satisfy $\prob{b \neq 0 } = \alpha$, where $\alpha \in [0,1)$ is the nominal fraction of samples the adversary is allowed to corrupt.

In what follows we denote by $\cP$ the distribution over non-corrupted data samples $(x,y)$ (i.e., $b=0$), and by $\cQ$ the distribution over corrupted samples (for which $b\neq0$). Thus, the contaminated samples in the \mdlref{simple_model} are coming from the mixture $(x,y)\sim (1-\alpha)\cP + \alpha \cQ$. This model is known as the 
$\alpha$-Huber contamination model~\citep{huber1964}.

\paragraph{Robust Linear Regression:} A learner is given $T$ independent samples $(x_i,y_i)_{i \in \{1,2,\dots,T\}}$ from \mdlref{simple_model}, and is required to learn a parameter vector $w\in\R^d$ for either the \textit{prediction} problem with respect to the \emph{non-corrupted} data\footnote{Note we can equivalently write\\ $\expc{(x,y)\sim\cP}{\left(\inn{w,x} - y\right)^2} =\expc{x,\epsilon}{\left(\inn{w,x} - y\right)^2}.$}
\begin{align} \label{eq:ExpectedLoss}
    \min_w F(w) := \expc{(x,y)\sim\cP}{\half\left(\inn{w,x} - y\right)^2}~,
\end{align}
or the \textit{estimation} problem, 
$$\argmin_w \norm{w - \sw}^2~,$$
where $\|\cdot\|$ denotes the $\ell_2$ norm.

In the ordinary prediction or estimation problems, when there is no adversary ($\alpha=0$), a standard solution method is \textit{least squares} \citep{gauss1809least}, in which the vector $w$ which minimizes an empirical version of either the prediction error or estimation error. However, this method is known to be fragile for $\alpha \neq 0$; As the next simple example demonstrates, for any $\alpha \in (0,1)$ the adversary can make the prediction error arbitrarily large, even for infinite number of samples.

\begin{example}\label{exmp:1d}
Let $\sw\in\R$, and
    assume $\epsilon = 0$ with probability $1$, and $x$ is distributed uniformly over $\{0,2\}$. Also consider the following
     adversary,
    $$b_i :=
    \begin{cases}
        \frac{C}{\alpha} &, \text{with probability } \alpha\\
        0 &, \text{otherwise.}
    \end{cases}~, \forall i\in[1,\ldots, T]~.
    $$
    Then $y_i = x_i \cdot \sw + b_i$ and on the population level (with infinite number of samples) the expected $L_2$ loss is 
    $$
        F(w) = (w-\sw)^2 + \expc{}{b}\cdot(w-\sw) + \expc{}{b^2}
    $$
    whose minimal value is attained for $w=\sw + \ofrac{2}\expc{}{b}=\sw + \ofrac{2}C$, which might be arbitrarily far from $\sw$.
\end{example}
We make the following assumptions throughout the paper:
\begin{assumption}[Bounded Parameter Vector]\label{asmp:w} $\|\sw\|\leq D$,
    $$\sw \in \cW := \left\{ w \in \R^d\ :\ \norm{w} \leq D\right\}~.$$
\end{assumption}
\begin{assumption}[Bounded Zero-Mean Noise]\label{asmp:epsilon}
    $|\epsilon| \leq \sigma$ with probability $1$ and $\expc{}{\epsilon} = 0$.
\end{assumption}
\begin{assumption}[Bounded Feature Vector]\label{asmp:x}
    $\norm{x} \leq 1$ with probability 1.
\end{assumption}
In \sectref{SubGaussian} we also extend to the case where the features and measurement noise are sub-Gaussian.

We denote by $\Sigma$ covariance matrix of the feature vector,
$$\Sigma := \expc{}{\left(x - \Ex\right)\left(x - \Ex\right)^T}~.$$
In later sections, in order to obtain guarantees on the estimation error we would also assume the following,
\begin{assumption}[Strictly Positive Definite Covariance Matrix of the Feature Vector]\label{asmp:pd_cov}
    $\Sigma \succeq \rho \cdot I$, where $\rho > 0$ and $I$ is the identity matrix.
\end{assumption}
\paragraph{Additional Definitions \& Notations:} We will use \asmpref{pd_cov} to show that the expected loss is a strongly convex function in $\cW$, and this property will be extensively used. So, we remind the reader the following properties:
\begin{property}\label{prop:sc_scnd}
    Let $f$ be twice continuously differentiable. $f$ is $\lambda$-strongly convex in $\cW$ if and only if,
    $$\forall w \in \cW :\ \g^2 f(w) \succeq \lambda I~.$$
\end{property}
\begin{property}\label{prop:sc_opt}
    Let $f$ be $\lambda$-strongly convex over $\cW$. Denote $\wopt := \argmin_{w \in \cW} f(w)$. Then, $\forall w \in \cW$,
    $$f(w) - f(\wopt) \geq \frac{\lambda}{2} \|w - \wopt\|^2~.$$
\end{property}

We also denote the orthogonal projection onto $\cW$ by $\proj{\cW}{\cdot}$, i.e.,
$\proj{\cW}{u}:=\argmin_{w\in\cW}\|w-u\|$.
\section{Huber Loss For Robust Regression}\label{sect:huber}

As stated in the previous section, the goal of the learner is either to find a good predictor or an accurate estimator under \mdlref{simple_model}. The straightforward approach to doing so is to minimize the $\ell_2$ loss while using the contaminated data, but as \exmpref{1d} shows this might completely fail. 
This is mainly due to the sensitivity of the $\ell_2$ loss to outliers.

Our approach is based on utilizing a robust loss function and minimizing it using the contaminated data. Concretely, we employ the Huber loss \citep{huber1964} which is known to have the two desired properties: it behaves similarly to the $\ell_2$ loss in a region around the origin, yet far away from the origin it behaves similarly to the $\ell_1$ loss and its gradients are bounded over $\R$. 

The Huber loss $h_R:\R \mapsto \R$ is a convex function, parameterized by a radius parameter $R$. It is defined as follows,
\begin{equation*}\label{eq:huber}
    \hubR{s} :=
        \begin{cases}
            \half s^2 &\text{, if } |s| \leq R\\
            R (|s| - \half R) &\text{, otherwise}
        \end{cases}
        ~.
\end{equation*}
We will also denote
$$\phiR{s} := \min\left\{ R, \max\left\{ s, -R \right\} \right\}~.$$
Note that $\nabla_w \hubR{\inn{a,w}} =\phiR{\inn{a,w}} \cdot a$.
Also note that $\|\phiR{s}\| \leq R$.

Using the Huber loss is a popular technique in robust regression, see e.g.~\citet{ts2014,pmlr-v139-d-orsi21a}. Nevertheless, in contrast to previous works, our analysis is done under a broader (less restrictive) set of assumptions, and our algorithm is a simple variant of SGD.

\paragraph{Choice of $R$:}
The radius parameter $R$ can be adjusted by the learner, and involves a trade-off between unnecessarily clipping clean gradients vs. limiting corrupted gradients. Our choice of $R$ is made such that the loss of non-corrupted observations is in the quadratic regime (so its gradient is not affected), while the loss for highly corrupted observations is in the linear regime (so its gradient's norm will be bounded). Next we detail on how to choose $R$.

Note that for a non-corrupted sample $i\ (b_i = 0)$ and $w\in \cW$, 
\begin{align*}
    \left| \inn{w,x_i} - y_i \right| &\leq \left| \inn{w, x_i} \right| + \left| \inn{\sw, x_i}\right| + |\epsilon_i|\\
    &\leq \|x_i\| \cdot \left( \|w\| + \|\sw\| \right) + |\epsilon_i|\\
    &\leq \|w\| + \|\sw\| + \sigma~,
\end{align*}
and we limit our search to $\cW$ since we know that $\sw \in \cW$. Thus, by taking
$R:=6D+\sigma$ we obtain that the gradient of non-corrupted samples is unaffected\footnote{It suffices to take $R=2 D+\sigma$, but we slightly increase the radius for consistency with derivations in \sectref{PD}.}. With this choice, for any $w\in\cW$ and non-corrupted sample $(x_i,y_i)$ we have \begin{align}\label{eq:UncorruptedHuberEll2}
    \hubR{\inn{w,x_i} - y_i} = \half \left(\inn{w,x_i} - y_i\right)^2~,
\end{align}
with probability 1.

\paragraph{Robust Objective Function.} We can now introduce the definition of \emph{expected Huber loss} with respect to the contaminated data,
$$
    L_R(w) := \expc{x,\epsilon,b}{\hubR{\inn{w, x} - y}}~,
$$
and note that this definition implies that $(x,y)$ are sampled from the contaminated \mdlref{simple_model}. Since $L_R(w)$ is defined with respect to the contaminated data distribution we can efficiently compute unbiased estimates of its gradients and apply SGD. We will use $L_R(w)$ as a proxy to the true expected loss $F(w)$ appearing in \eqnref{ExpectedLoss}.

The next Lemma shows that for the right choice of $R$ we can relate  $L_R(\cdot)$ to $F(\cdot)$,
\begin{lemma}\label{lem:LFH}
    Let $R = 6D + \sigma$. Then, $\forall w \in \cW$,
    $$L_R(w) = (1 - \alpha) F(w) + \alpha H(w)~,$$
    where $H(w) := \cexpc{x,\epsilon,b}{\hubR{\inn{w - \sw, x} - \epsilon - b}}{b \neq 0}$ represents the expected Huber loss of corrupted samples.
\end{lemma}
\begin{proof}
    Take $w \in \cW$. We use the law of total expectation with respect to $b$ on $L_R(w)$. We start with conditioning on $b = 0$: For $R$ as stated in the lemma,
    \begin{align*}
        &\cexpc{x,\epsilon,b}{\hubR{\inn{w - \sw, x} - \epsilon}}{b=0}\\
        &= \cexpc{x,\epsilon,b}{\half \left(\inn{w - \sw, x} - \epsilon\right)^2}{b=0}\\
        &= \expc{x,\epsilon}{\half \left(\inn{w - \sw, x} - \epsilon\right)^2}\\
        &= F(w)~,
    \end{align*}
    where the first equality follows from our choice of $R$ and \eqnref{UncorruptedHuberEll2}. The second equality follows since $b, \epsilon$ and $x$ are statistically independent and last equality follows from the definition of $F(w)$.\\
    Then, by the law of total expectation
    \begin{align*}
        L_R(w) &= (1 - \alpha) \cexpc{x,\epsilon,b}{\hubR{\inn{w - \sw, x} - \epsilon}}{b=0}\\
        &\ \ \ \ +\alpha \cexpc{x,\epsilon,b}{\hubR{\inn{w - \sw, x} - \epsilon - b}}{b \neq 0}\\
        &= (1 - \alpha) F(w) + \alpha H(w)~.
    \end{align*}
\end{proof}
\section{The General Feature Covariance Matrix Case}\label{sect:PSD}

In this section, we make no assumptions on $\Sigma$, that is, allow it to have vanishing eigenvalues.
We show that there are two regimes, $\textbf{(i)}$ non-centered features, i.e., $\Ex \neq 0$ and $\textbf{(ii)}$ centered featues, i.e., $\Ex = 0$. For the first regime, we provide an example showing that obtaining a low prediction error is impossible, even for a known $\Ex$. If $\Ex = 0$, we show that a good predictor can be learned (\thmref{PSD_main}).
Note that our algorithm does not require the knowledge of the corruptions fraction $\alpha$.

\subsection{Low Prediction Error is Generally Impossible}

Let $\alpha = \half$. Consider two one-dimensional models with parameter vectors $\sw_1, \sw_2$, where $\sw_2 = - \sw_1=1$. For both models we assume $\epsilon = 0, x = 1$ with probability $1$; thus $\E[x]$ is known and equals $1$. The adversary chooses his corruptions for every model $m \in [1,2]$ as follows,
$$b^{(m)} :=
\begin{cases}
    -2 \cdot \sw_m &, \text{with probability } \alpha= 1/2\\
    0 &, \text{otherwise.}
\end{cases}
$$
Then $y^{(m)} = \sw_m + b^{(m)}$ and both $y^{(1)}$ and $y^{(2)}$ has the same probability distribution of $y^{(m)} = \pm \sw_m$ with probability $\alpha = \half$. Since the contaminated data has the same distribution in both cases, then a learner cannot distinguish between these models. Thus, this adversary can cause any learner to incur a fixed (non-decreasing) prediction error, irrespective of the number of available samples.

\subsection{A Prediction Error Bound for Centered Features}

We next show that if $\Ex = 0$, then a low prediction error can be achieved even if $\Sigma$ has vanishing eigenvalues.

Next we present our key lemma which shows that we can bound the true error by the error of the expected Huber loss.
\begin{lemma}\label{lem:FL}
    For any $w \in \cW$, the following applies,
    $$F(w) - F(\sw) \leq \ofrac{1 - \alpha} \left( L_R(w) - L_R(\sw) \right)$$
\end{lemma}
\begin{proof}
We first show that $H(w)\geq H(\sw)$ for all $w \in \R^d$. Indeed, for any $w \in \R^d:$
    \begin{align*}
        H(w) &:= \cexpc{x,\epsilon,b}{\hubR{\inn{w - \sw, x} - \epsilon - b}}{b \neq 0}\\
        &\overset{(a)}{\geq} \cexpc{\epsilon,b}{\hubR{\expc{x}{\inn{w - \sw, x} - \epsilon - b}}}{b \neq 0}\\
        &\overset{(b)}{=} \cexpc{\epsilon,b}{\hubR{\inn{w - \sw, \Ex} - \epsilon - b}}{b \neq 0}\\
        &\overset{(c)}{=} \cexpc{\epsilon,b}{\hubR{- \epsilon - b}}{b \neq 0}\\
        &= \cexpc{x,\epsilon,b}{\hubR{\inn{\sw - \sw, x} - \epsilon - b}}{b \neq 0}\\
        &= H(\sw)~,
    \end{align*}
    where $(a)$ follows from Jensen's inequality applied to the convex function $\hubR{\cdot}$, as well as from the independence assumption, $(b)$ follows from the linearity of the inner product, $(c)$ follows from the assumption $\Ex = 0$. Using $H(w)\geq H(\sw)$ together with \lemref{LFH}, gives $\forall w \in \cW$,
    \begin{align*}
        F(w) - F(\sw) &= \ofrac{1 - \alpha} \left( L_R(w) - L_R(\sw)\right)\\
        &\ \ \ \ - \frac{\alpha}{1 - \alpha} \left(H(w) - H(\sw)\right)\\
        &\leq \ofrac{1 - \alpha} \left( L_R(w) - L_R(\sw)\right)~.
    \end{align*}
\end{proof}

\begin{algorithm}[tb]
   \caption{Huber SGD}
   \label{algo:lipsgd}
\begin{algorithmic}
   \State {\bfseries Input:} $D > 0, R > 0, \cW \subset \R^d, T \in \N_+$
   \State $w_1 := 0$
   \State $\eta := \frac{D}{R \sqrt{T}}$
   \For{$t=1$ {\bfseries to} $T$}
        \State Draw $(x_t, y_t)$ from (the contaminated) \mdlref{simple_model}
        \State $g_t := \phiR{\inn{w_t, x_t} - y_t} \cdot x_t$
        \State $w_{t+1} := \proj{\cW}{w_t - \eta g_t}$
   \EndFor
   \State {\bfseries Return:} $\bw := \meanT w_t$
\end{algorithmic}
\end{algorithm}
Thus, \lemref{FL} implies that by applying SGD to the expected Huber loss $L_R(\cdot)$ (while using the contaminated data), we can can obtain guarantees for the true expected prediction error.
This is exactly what we do in \algoref{lipsgd}. Its guarantees are formalized in the next theorem,
\begin{theorem}\label{thm:PSD_main}
    Let $\bw$ be the output of \algoref{lipsgd} after observing $T$ samples from \mdlref{simple_model}. Then,
    $$\expc{}{F(\bw)} - F(\sw) \leq \frac{R D}{(1 - \alpha) \sqrt{T}}$$
\end{theorem}
\begin{proof}
    Note that $L_R(\cdot)$ is convex and $\cW$ has a finite diameter $D$. Also, note that for a given $w_t\in\cW$ then 
    $g_t: = \phiR{\inn{w_t, x_t} - y_t} \cdot x_t$ is an unbiased estimate for the gradient of $L_R(\cdot)$ at $w_t$. 
    Moreover, since $\phiR{\cdot}$ is bounded by $R$ and the features are bounded, then the norm of $g_t$ is bounded by $R$. Thus, applying projected SGD as is done in \algoref{lipsgd} ensures that (see e.g. \citet[Theorem 14.8]{shalev2012}),
    $$\expc{}{L_R(\bw)} - L_R(\sw) \leq \frac{R D}{\sqrt{T}}~.$$
    Combining this with \lemref{FL} concludes the proof.
\end{proof}
\section{The Strictly Positive Definite Feature Covariance Matrix Case}\label{sect:PD}
Here we assume that $\Sigma\succeq\rho\cdot I$ for some known, strictly positive, $\rho>0$. In contrast to the general case discussed in the previous section, here we show that it is possible to obtain guarantees even if $\Ex\neq 0$. We also establish faster convergence rates compared to the general case.

We first consider the case in which the expectation of the features $\Ex$ is known to the learner, so it can perfectly center the features. We show that feeding these centered samples to an appropriate variant of SGD leads to an accurate estimation of $\sw$ with an error rate of $\bigO{\onicefrac{(1 - \alpha)^2 T}}$ (\thmref{known_main}).

We then show that the same estimation rate can be achieved even if the expectation of the features is unknown to the learner, but rather estimated from data. In~\thmref{unknown_main} we show that this can be done at the same rate as in the known expectation case up to logarithmic factors in $T$.

\subsection{The Known Expectation Case}\label{sect:known}
When the feature vector is not centered, the minimizer of $L_R(\cdot)$ might be different from $\sw$. A natural way to avoid it is to center the features, i.e. $x_i - \Ex$. In fact, \mdlref{simple_model} can be written with centered feature vectors as follows,
\begin{model} \label{mdl:centered_model}
\qquad \qquad \qquad
    $y = \inn{\sw, x - \Ex} + \inn{\sw,\Ex} + \epsilon + b~,$
\end{model}
which is still a linear model albeit with two differences. First, the norm of the centered features might be larger than $1$. To resolve this we recall that the radius parameter of the Huber loss was set to $R= 6 D + \sigma$, and so for any non-corrupted sample in \mdlref{centered_model}
\begin{align*}
    \left| \inn{w,x - \Ex} - y \right| &\leq D \left( \|x\| + \left\|\Ex\right\| \right) + \left| y \right|\\
    &\leq 2 D + D + \sigma
    \leq R~,
\end{align*}
where here, the first inequality follows from the triangle inequality and the second one follows from \asmpref{x}.
Thus, our choice of $R$ ensures that for any $w\in\cW$ and non-corrupted sample $(x_i,y_i)$ we have, $\hubR{\inn{w,x_i- \Ex} - y_i} = \half \left(\inn{w,x_i- \Ex} - y_i\right)^2$ with probability 1.

A second difference in \mdlref{centered_model}, is that it has an additional unknown quantity, to wit $\inn{\sw,\Ex}$.
As we next show, this is inconsequential.
With slight abuse of notation define,
\begin{align}\label{eq:LRCentered}
    L_R(w) = \expc{x,\epsilon,b}{\hubR{\inn{w, x - \Ex} - y}}~.
\end{align}
This expected Huber loss of centered features is also minimized by $\sw$:
\begin{lemma}\label{lem:L_R_sc}
    $L_R(w)$ is $(1 - \alpha) \rho$-strongly convex in $\cW$ and $\sw = \argmin_{w \in \cW} L_R(w)$.
\end{lemma}
\begin{proof}[Proof Sketch]
    In a high level, we show that $L_R(\cdot)$ can be decomposed to a sum of a $(1 - \alpha) \rho$-strongly convex function and a convex function, and so it is a $(1 - \alpha) \rho$-strongly convex function.
    This decomposition is similar to the one in \lemref{LFH}.
    The optimality of $\sw$ is also done similarly to \lemref{LFH}. The full proof can be found in \sectref{known_apnd}, which establishes \lemref{G_R_sc_apnd}, a more general version of this lemma.
\end{proof}

\paragraph{Huber Loss SGD with Centered Features (\algoref{known}):}
\lemref{L_R_sc} is instrumental in achieving fast convergence of the Huber loss, and in turn for the estimation error $\|\bw~-~\sw\|^2$.
As the lemma implies $L_R(\cdot)$ is strongly convex and maintains the same global optimum $\sw$ as the true expected loss $F(\cdot)$. Thus, it is natural to apply an appropriate version of SGD for strongly-convex functions to $L_R(\cdot)$ while using the contaminated data with \emph{centered} features, in order to obtain guarantees to $\|\bw - \sw\|^2$. This is exactly what we do in \algoref{known}. 
Concretely, \algoref{known} utilizes a SGD with $\half$-suffix averaging with the following guarantees:
\begin{lemma}[{\citet[Theorem 5]{rakhlin2012making}}]\label{lem:sgd_5}
    Consider SGD with $\half$-suffix averaging and with step size $\eta_t := \nicefrac{1}{\lambda t}$. Suppose $f$ is $\lambda$-strongly convex, and that $\expc{}{\norm{g_t}^2} \leq G^2$ for all $t$. Then for any $T$, it holds that $\forall w \in \cW$
    $$\expc{}{f(\bw)} - f(w) \leq \frac{9 G^2}{\lambda T}~.$$
\end{lemma}

\begin{algorithm}[tb]
   \caption{Huber SGD for known expectation}
   \label{algo:known}
\begin{algorithmic}
   \State {\bfseries Input:} $R > 0, \lambda > 0, \cW \subset \R^d, T \in \N_+, \Ex$ 
   \State $w_1 := 0$
   \For{$t=1$ {\bfseries to} $T$}
        \State Draw $(x_t, y_t)$ from (the contaminated) \mdlref{simple_model}
        \State $\eta_t := \onicefrac{\lambda t}$
        \State $g_t := \phiR{\inn{w_t, x_t - \Ex} - y_t} \cdot \left( x_t - \Ex \right)$
        \State $w_{t+1} := \proj{\cW}{w_t - \eta_t g_t}$
   \EndFor
   \State {\bfseries Return:} $\bw := \frac{2}{T} \sum_{t=1 + \nicefrac{T}{2}}^{T} w_t$
\end{algorithmic}
\end{algorithm}

Now, by using \lemref{L_R_sc} we can bound the estimation error of \algoref{known} as follows,
\begin{theorem}\label{thm:known_main}
    Let $\bw$ be the output of \algoref{known} with input $\left(R, (1-\alpha)\rho, \cW, T, \Ex \right)$, then,
    $$
    \expc{}{\| \bw - \sw \|^2} \leq 
    \frac{72 R^2}{((1 - \alpha) \rho)^2 \cdot T}~.
    $$
\end{theorem}

\begin{proof}
    Note that the $g_t$ that we employ in \algoref{known} is an unbiased gradient estimate for the Huber loss $\g L_R(w_t)$ (recall the definition in \eqnref{LRCentered}).
    Also, since $\phiR{\cdot}$ is bounded by $R$ and the feature norms are bounded by $1$, then 
    $\|g_t\|\leq 2R~,\forall t$.
    Moreover, \lemref{L_R_sc} implies that $L_R(\cdot)$ is $(1-\alpha)\rho$-strongly-convex with optimum $\sw\in \cW$. 
    
    Now, since \algoref{known} actually applies strongly-convex SGD with $\half$-suffix averaging, then \lemref{sgd_5} with $G~:=~2R$ yields,
    $$
        \expc{}{L_R(\bw)} - L_R(\sw) \leq \frac{36 R^2}{((1 - \alpha)\rho) \cdot T}~.
    $$
    Combining the above with the strong-convexity of $L_R(\cdot)$ while using \propref{sc_opt} establishes the theorem.
\end{proof}

\subsection{The Unknown Expectation Case}\label{sect:unknown}
In this section we present \algoref{unknown} which generalizes what we did in the previous section for to the case where $\Ex$ is unknown. We assume that $2T$ samples are provided to the algorithm. Similarly to the previous section, the learning algorithm is based on centering the features before feeding them to an SGD with $\half$-suffix averaging. The only difference is that in \algoref{unknown} we center the features using the empirical mean based on first $T$ samples,
$\mu~:=~\meanT~z_t~.$

Similarly to the previous section, with this definition, \mdlref{simple_model} can be written as follows,
\begin{model}\label{mdl:uknown_bounded}
\qquad \qquad \qquad 
    $ \qquad y = \inn{\sw, x - \mu} + \inn{\sw,\mu} + \epsilon + b~.$
\end{model}
Our proposed \algoref{unknown} is based on sample splitting, and has two phases: In the first phase, $\mu$ is computed using $T$ samples, and in the second phase, $T$ steps of SGD are performed on fresh samples that are centered using $\mu$. 
Since the samples used in the SGD algorithm are independent of the estimator $\mu$ (which is a function of different samples), the analysis of the SGD algorithm can be made conditionally on $\mu$, without affecting the distribution of the samples.

\begin{algorithm}[tb]
   \caption{Huber SGD for unknown expectation}
   \label{algo:unknown}
\begin{algorithmic}
   \State {\bfseries Input:} $R > 0, \lambda > 0, \cW \subset \R^d, T \in \N_+$
   \State {\bfseries Phase 1: Compute $\mu$}
       \State Draw $T$ samples $\{(z_i, y_i)\}_{i=1}^T$ from (the contaminated) \mdlref{simple_model}, and estimate the mean,
       \State $\mu := \onicefrac{T} \sum_{i=1}^T z_i$
   
   \State {\bfseries Phase 2: SGD with $\half$-suffix averaging given $\mu$}
       \State $w_1 := 0$
       \For{$t=1$ {\bfseries to} $T$}
            \State Draw $(x_t, y_t)$ from (the contaminated) \mdlref{simple_model}
            \State $\eta_t := \nicefrac{1}{\lambda t}$
            \State $g_t := \phiR{\inn{w_t, x_t - \mu} - y_t} \cdot \left( x_t - \mu \right)$
            \State $w_{t+1} := \proj{\cW}{w_t - \eta_t g_t}$
       \EndFor
   \State {\bfseries Return:} $\bw := \frac{2}{T} \sum_{t=1 + \nicefrac{T}{2}}^{T} w_t$
\end{algorithmic}
\end{algorithm}

With slight abuse of notation define
$$L_R(w) = \expc{x,\epsilon,b}{\hubR{\inn{w - \sw, x - \mu} - \inn{\sw, \mu} - \epsilon - b}}~.$$
\begin{lemma}\label{lem:uknown_L_sc}
    $\forall \mu \in \R^d :\ L_R(w)$ is $(1 - \alpha) \rho$-strongly convex in $\cW$.
\end{lemma}
\begin{proof}
    Because $\mu$ is given, the proof is immediate from the more general \lemref{G_R_sc_apnd} with $v = q = \mu$.
\end{proof}
Similarly to the case of known expectation, \algoref{unknown} applies a strongly-convex variant of SGD to the above defined $L_R(\cdot)$.
Unfortunately, In contrast to the case of known expectation, the optimum of $L_R(\cdot)$ in $\cW$ \emph{is not necessarily $\sw$}.
Nevertheless, we are still able to establish a fast convergence rate for \algoref{unknown} as can be seen below,
\begin{theorem}\label{thm:unknown_main}
    Let $\bw$ be the output of \algoref{unknown} with input $\left(R, (1-\alpha)\rho, \cW, T \right)$, then,
    $$
        \expc{}{\norm{\bw - \sw}^2} \leq \frac{72 R^2 \cdot (2\log T + 1)}{((1 - \alpha) \rho)^2 \cdot T}~.
    $$
\end{theorem}
\paragraph{Analysis.} Prior to proving  \thmref{unknown_main} we need to introduce some definitions and establish some auxiliary lemmas.

\begin{lemma}\label{lem:huber_lip}
    $h_{R}(\cdot)$ is $2R$-Lipschitz with probability 1.
\end{lemma}
\begin{proof} For any $w \in \cW$
    \begin{align*}
        &\norm{\g \hubR{\inn{w - \sw, x - \Ex} - \inn{\sw,\mu} - \epsilon - b}}\\
        &= \norm{\phiR{\inn{w - \sw, x - \Ex} - \inn{\sw,\mu} - \epsilon - b} \cdot \left(x - \Ex\right)}\\
        &\leq 2 R~,
    \end{align*}
    where the equality follows from definition and the inequality follows from the definition of $\phiR{\cdot}$ and \asmpref{x} with probability 1.
\end{proof}

The following is a version of Hoeffding's inequality for random vectors,
\begin{lemma}[{\citet[Theorem 2.1]{Kakade13}}]\label{lem:hoeff_vec}
    Assume that $\{x_i\in\R^d\}_{i=1}^T$ are random variables sampled i.i.d and $\| x_i \| \leq K$ almost surly. Then with probability greater than $1 - \delta$,
    $$\left\| \frac{1}{T}\sum_{i=1}^Tx_i - \Ex \right\| \leq 6 K \sqrtdeltaT~.$$
\end{lemma}

The main challenge is that now $\sw$ is not necessarily the optimum of $L_R(\cdot)$ in $\cW$. 
To circumvent this, we will analyze an auxiliary expected loss function which we define below,
$$
    \tL_R(w) := \expc{x,\epsilon,b}{\hubR{\inn{w - \sw, x - \Ex} - \inn{\sw, \mu} - \epsilon - b}}~.
$$
This function is not the one minimized by the algorithm, however, it has the following desirable property:

\begin{lemma}\label{lem:tL_sc}
    $\forall \mu \in \R^d :\ \tL_R(w)$ is $(1 - \alpha) \rho$-strongly convex in $\cW$ and $\sw = \argmin_{w \in \cW} \tL_R(w)$.
\end{lemma}
\begin{proof}
    Because $\mu$ is given, the proof is immediate from the more general \lemref{G_R_sc_apnd} with $v = \Ex$ and $q = \mu$.
\end{proof}

Now, we can use \propref{sc_opt} to bound the estimation error: for a given $\mu$,
\begin{align*}
    \frac{(1 - \alpha) \rho}{2} \| \bw - \sw \|^2 &\leq \tL_R(\bw) - \tL_R(\sw)\\
    &= \tL_R(\bw) - L_R(\bw)\\
    &\ \ \ \ + L_R(\bw) - L_R(\sw)\\
    &\ \ \ \ + \underbrace{L_R(\sw) - \tL_R(\sw)}_{=0}~,
\end{align*}
where the last term equals zero by the definitions of $L_R(\sw)$ and $\tL_R(\sw)$. So, By taking expectation we obtain,
\begin{equation}\label{eq:bw_sw}
    C \cdot \expc{}{\| \bw - \sw \|^2} \leq \expc{}{\tL_R(\bw) - L_R(\bw)} + \expc{}{L_R(\bw) - L_R(\sw)}~,
\end{equation}
where $C = \frac{(1 - \alpha) \rho}{2}$. We now bound each of these terms. 

For the second term, since \algoref{unknown} applies SGD to the strongly-convex function $L_R(\cdot)$, then the expectation of $L_R(\bw) - L_R(\sw)$ is upper bounded by $\bigO{\ofrac{T}}$ due to \lemref{sgd_5}, with the parameter choice $F~=~L_R, G~=~R, \lambda~=~(1~-~\alpha)~\rho$.

The first term above can be upper bounded as follows,
\begin{lemma}\label{lem:tL_LR}
    Let  $C = \frac{(1 - \alpha) \rho}{2}$. Then for a given $\mu$,
    $$
        \expc{}{\tL_R(\bw) - L_R(\bw)} \leq \frac{C}{2} \cdot \expc{}{\|\bw - \sw \|^2} + \frac{2 R^2}{C} \cdot \frac{36 \log T + 4}{T}
    $$
\end{lemma}
\begin{proof}
    Take $C = \frac{(1 - \alpha) \rho}{2}$. Then,
    \begin{align}
    \label{eq:Ineq1LemA4}
        \tL_R&(\bw) -  L_R(\bw) \nonumber\\
        &\overset{(a)}{=} \underset{x,\epsilon,b}{\E} \Big[\hubR{\inn{\bw - \sw, x - \Ex} - \inn{\sw, \mu} - \epsilon - b} - \hubR{\inn{\bw - \sw, x - \mu} - \inn{\sw, \mu} - \epsilon - b}\Big]\nonumber\\
        &\overset{(b)}{\leq} 2 R \cdot \abs{\inn{\bw - \sw, \mu - \Ex}}\nonumber\\
        &\overset{(c)}{\leq} C \cdot \frac{2 R}{C} \cdot \norm{\bw - \sw} \cdot \norm{\mu - \Ex}\nonumber\\
        &\overset{(d)}{\leq} \frac{C}{2} \norm{\bw - \sw}^2 + \frac{2 R^2}{C} \norm{\mu - \Ex}^2~,
    \end{align}
    where $(a)$ follows from the definitions of $\tL_R(\cdot)$ and $L_R(\cdot)$, $(b)$ follows from \lemref{huber_lip}, $(c)$ follows from the Cauchy-Schwarz's inequality and $(d)$ follows from Young's inequality: $a \cdot b \leq \half (a^2 + b^2)$ where $a~=~\|\bw - \sw\|$ and $b~=~\frac{2 R}{C}~\cdot~\norm{\mu - \Ex}$.
    
    Note that $\bw$ and $\mu$ are random as they are that depend on  $(z_i,y_i)_{i \in [1,2,\dots,T]}$ and $(x_t,y_t)_{t \in [1,2,\dots,T]}$. By taking an expectation with respect to the $2T$ samples on both sides we have
    $$\expc{}{\tL_R(\bw) - L_R(\bw)} \leq \expc{}{\frac{C}{2} \norm{\bw - \sw}^2 + \frac{2 R^2}{C} \norm{\mu - \Ex}^2}~.$$
    We conclude the proof by bounding $\expc{}{\norm{\mu - \Ex}^2}$. We make use of \lemref{hoeff_vec} by defining $\cB$ as the event for which $\norm{\mu - \Ex} \leq 6 \sqrt{\frac{\logdelta}{T}}$. Then,
    \begin{align*}
        \expc{}{\norm{\mu - \Ex}^2} &\overset{(a)}{=} \expc{z_1,z_2,\ldots,z_T}{\norm{\mu - \Ex}^2}\\
        &\overset{(b)}{\leq} \prob{\cB} \cdot \cexpc{z_1,z_2,\ldots,z_T}{\norm{\mu - \Ex}^2}{\cB} + \prob{\cB^c} \cdot \cexpc{z_1,z_2,\ldots,z_T}{\norm{\mu - \Ex}^2}{\cB^c}\\
        &\overset{(c)}{\leq} (1 - \delta) \cdot \frac{36 \logdelta}{T} + \delta \cdot 4\\
        &\overset{(d)}{\leq} \frac{36 \log T + 4}{T}~,
    \end{align*}
    where $(a)$ follows from the i.i.d assumption on the features, $(b)$ follows from the law of total expectation, $(c)$ follows from the definition of $\cB$ and \lemref{hoeff_vec} (which is true for every $\delta \in (0, 1)$) and a naive upper bound of $\norm{\mu - \Ex} \leq 2$ with probability 1 (which follows from \asmpref{x}) and $(d)$ follows from taking $\delta = \ofrac{T}$. Plugging the above into \eqnref{Ineq1LemA4} concludes the proof.
\end{proof}
We are now ready to prove \thmref{unknown_main}.
\begin{proof}[Proof of \thmref{unknown_main}]
    By plugging \lemref{sgd_5} and \lemref{tL_LR} back into \eqnref{bw_sw}, we obtain,
    $$
        C \cdot \expc{}{\norm{\bw - \sw}^2} \leq
        \frac{9 R^2}{(1 - \alpha) \rho T} + \frac{C}{2} \cdot \expc{}{\norm{\bw - \sw}^2} + \frac{2 R^2}{C} \cdot \frac{36 \log T + 4}{T}~.
    $$
    Recalling $C = \frac{(1 - \alpha) \rho}{2}$, the above implies, 
    $$
        \expc{}{\norm{\bw - \sw}^2} \leq \frac{72 R^2 \cdot (2\log T + 1)}{((1 - \alpha) \rho)^2 \cdot T}~.
    $$
    which establishes the theorem.
\end{proof}
\section{Extensions}\label{sect:extensions}

Here we present two extensions for the case where $\Sigma\succeq \rho \cdot I$:
In \sectref{SubGaussian} we go beyond the assumptions of bounded features \& noise, and extend our results to the sub-Gaussian case. In \sectref{adaptivity}, we point out to an algorithm that does not require knowledge of the fraction of contaminated samples $\alpha$, but rather implicitly adapts to it. 
This improves over Algorithms~\ref{algo:known} \& \ref{algo:unknown} which require $\alpha$.

\subsection{Sub-Gaussian noise and feature vectors}\label{sect:SubGaussian}
We next relax the assumptions that the noise and the norm of the feature vectors are bounded with probability $1$, and replace them with the following sub-Gaussian assumptions.
\begin{assumption}[Sub-Gaussian Feature Vector]\label{asmp:x_subG}
    The feature vector $x$ is sub-Gaussian with variance proxy $\kappa^2$, that is $\prob{\norm{x}>u}\leq2e^{-\frac{u^2}{2\kappa^2}}$ for all $u\in\R$.
\end{assumption}
\begin{assumption}[Sub-Gaussian Noise]\label{asmp:epsilon_subG}
    The noise $\epsilon$ is sub-Gaussian with variance proxy $\sigma^2$, that is $\prob{|\epsilon|>u}\leq2e^{-\frac{u^2}{2\sigma^2}}$ for all $u\in\R$.
\end{assumption}
\asmpref{x_subG} replaces \asmpref{x} and \asmpref{epsilon_subG} replaces \asmpref{epsilon}.
We show the following,
\begin{theorem} \label{thm:SGD for sub-Guassian}
    Under Assumptions \ref{asmp:x_subG} \& \ref{asmp:epsilon_subG}, define, $R~=~C\cdot(\kappa~D~+~\sigma)~\sqrt{\log~T}$, where $C>0$ is an explicit constant, which depends logarithmically on $\kappa,\rho,\sigma$.
    Denote $\bw$ as the output of \algoref{unknown} with input $\left(R, (1-\alpha)\rho, \cW, T \right)$, then, 
    $$
        \expc{}{\norm{\bw - \sw}^2} \leq \left(\frac{D^2}{14} + \frac{288 R^2 \cdot (2\log T + 1)}{((1 - \alpha) \rho)^2}\right) \cdot \ofrac{T}~.
    $$
\end{theorem}

\begin{proof}[Proof sketch]
    We denote by $\cF^c$ the event in which both the features and the noise values are bounded by some appropriately chosen constants $u_1,u_2$, that is
    $$
        \cF^c := \left\{ \underset{i \in [1,2,\dots,T]}{\bigcap} \|x_i\| \leq u_1 \cap |\epsilon_i| \leq u_2\right\}~.
    $$
    The constants $u_1, u_2$ are chosen such that $\prob{\cF}~\leq~\delta~:=~\min\left\{\half, \frac{\rho^{2}}{105^{2}\kappa^{4}}\cdot\ofrac{T^{2}}\right\}$. Then, by the law of total expectation to decompose the expected estimation error:
    \begin{align*}
        \expc{}{\norm{\bw - \sw}^2} &= \prob{\cF} \cdot \cexpc{}{\norm{\bw - \sw}^2}{\cF}\\
        &\ \ \ \ +\prob{\cF^c} \cdot \cexpc{}{\norm{\bw - \sw}^2}{\cF^c}~,
    \end{align*}
    The first term (conditioned on $\cF$) is bounded by Cauchy-Schwarz's inequality and a naive bound of $\norm{\bw - \sw} \leq 2D$. The second term, which conditions on $\cF^c$, and thus implies that the feature vector and the noise are bounded, can in principle be bounded using the analysis in the bounded setting. The main technical part is that it need to be established that the covariance matrix of the feature vector is still strictly convex, even after conditioning on the event $\cF^c$. To this end, \lemref{modifed feature covariance matrix minimal eigenvalue_apnd} shows that for any event $\cG$ such that $\prob{\cG}\leq\delta$ for some $\delta\in(0,\ofrac{2})$ it holds that
    \begin{align*}
        \tilde\Sigma&:=\cexpc{}{\left(x - \cexpc{}{x}{\cG^c} \right)\left(x - \cexpc{}{x}{\cG^c} \right)^T}{\cG^c}\\
        &\succeq \left(\rho - 105\kappa^2\sqrt{\delta}\right)\cdot I~.
    \end{align*}
    By choosing a sufficiently small $\delta$ we assure that $\tilde\Sigma~\succeq~\frac{\rho}{2} \cdot I$. Then, conditioned on $\cF^c$, and with a sufficiently large radius $R$ for the Huber loss, the analysis of \sectref{unknown} is in tact, and the convergence rate that was stated in \thmref{unknown_main} holds for this conditional expectation. Summing the bounds of both terms then results with the bound on the squared error. The full proof can be found in \thmref{SGD for sub-Guassian_apnd}.
\end{proof}

\subsection{Adaptivity to Contamination Fraction $\alpha$}
\label{sect:adaptivity}

In \sectref{PD} we have assumed that $\alpha$ is known in order to find a good estimation of $\sw$. Our derivation shows that the expected Huber loss $L_R(\cdot)$ is $(1-\alpha)\rho$-strongly-convex, and we encode this information into the learning rate of SGD with $\half$-suffix averaging that we employ as a part of Algorithms~\ref{algo:known} \& \ref{algo:unknown}, as these algorithms require the strong-convexity parameter in order to ensure fast convergence.

In practice, it is unrealistic to assume that the fraction of contamination $\alpha$ is known.
Fortunately, \citet{Cutkosky2018BlackBoxRF} have recently presented a novel and practical first order algorithm for stochastic convex optimization that enables to implicitly adapt to the strong-convexity of the problem at hand. So, if we apply this algorithm instead of SGD with $\half$-suffix averaging, then we immediately obtain adaptivity to both $\alpha$ and $\rho$.
The full description of the algorithm appears in Section $6$ of~\citet{Cutkosky2018BlackBoxRF}. 
\section{Experiments}\label{sect:experiments}

\begin{figure}[t!]
\centering
\begin{subfigure}[0]{.5\textwidth}
  \centering
  \includegraphics[width=\linewidth]{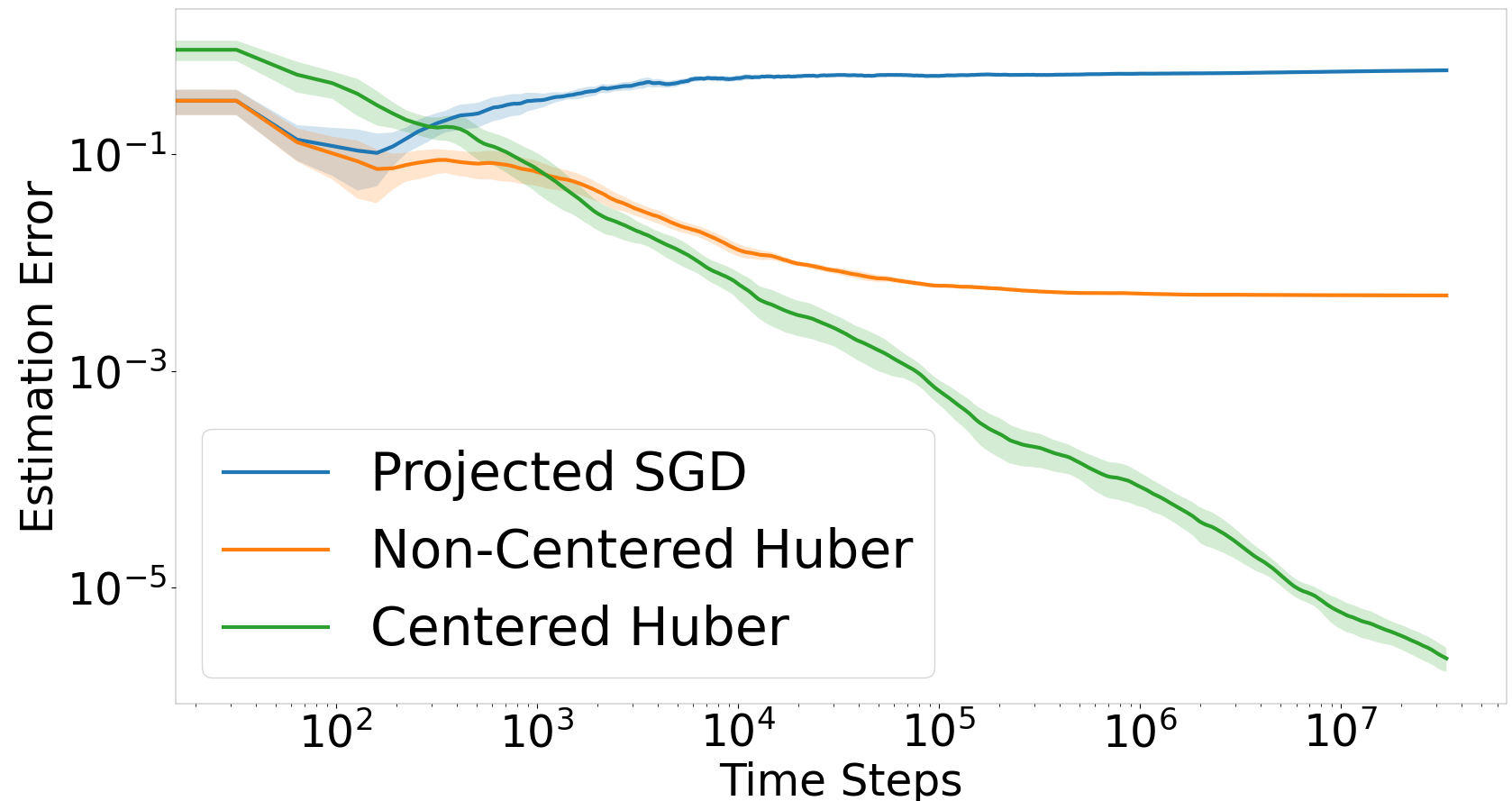}
  \caption{Results for $\alpha = 0.01$.}
  \label{expr2}
\end{subfigure}%
\begin{subfigure}[1]{.5\textwidth}
  \centering
  \includegraphics[width=\linewidth]{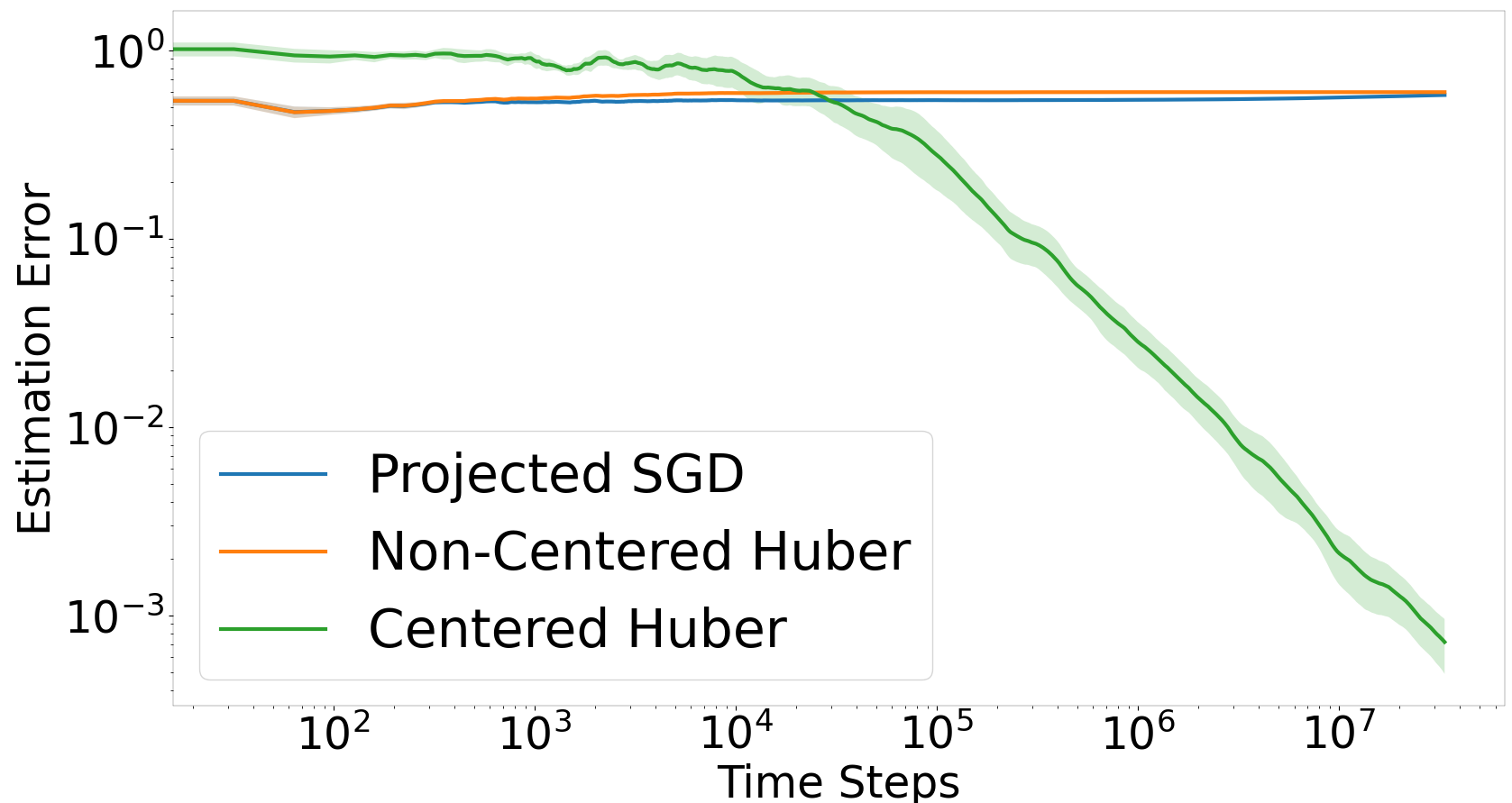}
  \caption{Results for $\alpha = 0.7$.}
  \label{expr1}
\end{subfigure}
\caption{}
\label{fig:exprs}
\end{figure}

We perform 2 experiments under the assumptions of \mdlref{uknown_bounded} in $d=5$ dimensions. 
The optimum $\sw$ is sampled from the unit ball and is the same for both experiments. We also let $\epsilon~\sim~\text{Uniform}[-0.1,0.1]$, and
$x~\sim~\text{Uniform}\left[-1/\sqrt{d}, 0\right]^d$ such that $\|x\| \leq 1$ w.p.~$1$.1 The adversary picks
$b=10^5$ with probability $\alpha$, and $0$ with probability $1-\alpha$. We test two cases $\alpha \in \{0.01,0.7\}$.

We compare $3$ algorithms: \textbf{(i)} Projected SGD over $\ell_2$ loss, \textbf{(ii)} Non-Centered Huber, which is Huber SGD without centering, and \textbf{(iii)} Centered Huber (\algoref{unknown}). 
All methods use the same learning rate $\eta_t: = \eta_0/t$ where $\eta_0 = 1/(1-\alpha)\rho$. Also all methods use the same number of samples.
In Centered Huber we compute the estimate of feature means on the fly rather than using extra samples (i.e., we use $\mu_t: = \onicefrac{t}\sum_{i=1}^{t-1} x_t$ instead of $\mu =\onicefrac{T}\sum_{t=1}^T z_t$). We repeat each experiment $15$ times and add confidence intervals. The experiments, shown in Figure~\ref{fig:exprs}, clearly demonstrate the benefit of our approach compared to the baselines.

\section{Conclusion}
In this paper, we have analyzed robust linear regression under general assumptions on the feature vectors, noise and the oblivious adversary. We have shown that low prediction error requires either centered features or strict positivity of the features vector covariance matrix, and provided efficient SGD-style algorithms and established error rate convergence bounds. Finally, we have provided SGD variants that do not require prior knowledge on the fraction of contamination. While the linear regression model is both basic and important, an important avenue for future work is to generalize the algorithms and the error bounds to more elaborated ML models. 

\subsection*{Acknowledgements}
The work on this paper was supported in part by the Israel Science Foundation (grant No. 447/20).

\newpage
\bibliography{bib}
\bibliographystyle{abbrvnat}

\newpage
\appendix
\section{Generalization of \lemref{L_R_sc}, \lemref{uknown_L_sc} and \lemref{tL_sc}}\label{sect:known_apnd}

\begin{lemma}\label{lem:G_R_sc_apnd}
    Given $v,q\in\R^d$ such that $\|v\|,\|q\| \leq 1$ and $R=6D + \sigma$, let
    $$G_R(w) := \expc{x,\epsilon,b}{\hubR{\inn{w - \sw, x - v} - \inn{\sw,q} - \epsilon,b}}~.$$
    Then, $G_R$ is $(1 - \alpha) \rho$-strongly convex.
    Furthermore, if $v = \Ex$ we have $\sw = \argmin_{w\in\cW} G_R(w)$.
\end{lemma}
\begin{proof}
    We show that $G_R(w)$ is a sum of a $(1 - \alpha) \rho$-strongly function and a convex function and as such it is a $(1 - \alpha) \rho$-strongly convex.
    
    Define
    $$H(w) := \cexpc{x,\epsilon,b}{\hubR{\inn{w - \sw, x - v} - \inn{\sw, q} - \epsilon -b}}{b \neq 0}~,$$
    which is convex as an average of convex functions. Also define
    \begin{align*}
        F(w) &:= \cexpc{x,\epsilon,b}{\hubR{\inn{w - \sw, x - v} - \inn{\sw, q} - \epsilon - b}}{b = 0}\\
        &\overset{(a)}{=} \cexpc{x,\epsilon,b}{\half \left( \inn{w - \sw, x - v} - \inn{\sw, q} - \epsilon \right)^2}{b = 0}\\
        &\overset{(b)}{=} \expc{x,\epsilon}{\half \left( \inn{w - \sw, x - v} - \inn{\sw, q} - \epsilon \right)^2}\\
        &= \expc{x,\epsilon}{\half \inn{w - \sw, x - v}^2}\\
        &\ \ \ \ + \expc{x,\epsilon}{\epsilon \cdot \inn{w - \sw, x - v} + \inn{\sw, q} \cdot \inn{w - \sw, x - v}} + S~,
    \end{align*}
    where $S$ is a constant independent of $w$. $(a)$ follows from the boundness assumptions on $x, v, q, \epsilon$ and $R = 6D + \sigma$ and $(b)$ follows from the assumption that $x,\epsilon,b$ are statistically independent.
    
    $F(w)$ is a polynomial and hence twice continuously differentiable. We take the second derivative and show it is positive definite,
    \begin{align*}
        \g^2 F(w) &= \expc{x}{(x - v)(x - v)^T}\\
        &= \expc{x}{(x - \Ex + \Ex - v)(x - \Ex + \Ex - v)^T}\\
        &= \expc{x}{(x - \Ex)(x - \Ex)^T + (\Ex - v)(\Ex - v)^T}\\
        &\ \ \ \ + \underbrace{\expc{x}{(x - \Ex)(\Ex - v)^T + (\Ex - v)(x - \Ex)^T}}_{=0}\\
        &= \Sigma + (\Ex - v)(\Ex - v)^T\\
        &\succeq \Sigma
    \end{align*}
    last equality follows from the definition of $\Sigma$ and the fact that $v$ is some given vector. First inequality follows from the fact that $(\Ex - \mu)(\Ex - \mu)^T$ is positive semi-definite by definition.
    
    So, \asmpref{x} with $\rho>0$ and \propref{sc_scnd} assure that $F$ is a $\rho$-strongly convex function.
    Then, by the law of total expectation with respect to $b$:
    $$G_R(w) = (1 - \alpha) F(w) + \alpha H(w)~.$$
    Since $(1 - \alpha) F(w)$ is $(1 - \alpha) \rho$-strongly convex, so is $G_R(w)$.
    
    If $v=\Ex$ we can also show that $\sw = \argmin_{w\in\cW} G_R(w)$. For any $w \in \cW$
    \begin{align*}
        G_R(w) &:= \expc{x,\epsilon,b}{\hubR{\inn{w - \sw, x - \Ex} - \inn{\sw, q} - \epsilon - b}}\\
        &\overset{(a)}{\geq} \expc{\epsilon,b}{\hubR{\expc{x}{\inn{w - \sw, x - \Ex} - \inn{\sw, q} - \epsilon - b}}}\\
        &\overset{(b)}{=} \expc{\epsilon,b}{\hubR{\inn{w - \sw, \expc{x}{x - \Ex}} - \inn{\sw, q} - \epsilon - b}}\\
        &= \expc{\epsilon,b}{\hubR{- \inn{\sw, q} - \epsilon - b}}\\
        &= \expc{\epsilon,b}{\hubR{\inn{\sw - \sw, x - \Ex} - \inn{\sw, q} - \epsilon - b}}\\
        &= G_R(\sw)~,
    \end{align*}
    where $(a)$ follows from convexity of the Huber loss and Jensen's inequality and $(b)$ follows from the linearity of the inner product.
    Moreover, $\sw$ is the unique minimizer of $G_R(w)$ in $\cW$. This is because according to \propref{sc_opt}, if $G_R(w)=G_R(\sw)$ for some  $w\in\cW$ then
    $$0 = G_R(\sw) - G_R(w) \geq \frac{(1 - \alpha) \rho}{2} \|w - \sw\|^2~.$$
    Since $\frac{(1 - \alpha) \rho}{2} > 0$ is assumed this implies $\wopt = \sw$.
\end{proof}

\newpage
\section{Proofs for \sectref{SubGaussian}}
\label{apnd:SubGaussian}

\begin{remark}
We have stated the sub-Gaussian assumptions (\asmpref{x_subG} and \asmpref{epsilon_subG}) in terms of the tails of the probability density functions. We refer the reader to \citet[Chapter 2]{vershynin2018highdim}, for equivalent definitions of sub-Gaussian variables in terms of the moment generating function or in terms of integer moments (all these definitions are essentially equivalent). 
\end{remark}

We will use the following bounds on the moments of sub-Gaussian random variables.
\begin{lemma}\label{lem:moments of subgaussian}
Let $z$ be a sub-Gaussian random variable with variance proxy $\lambda^2$, that is $\prob{|z|>u}\leq2e^{-\frac{u^2}{2\lambda^2}}$ for all $u\in\R$. Then, for any $p\geq 1$
$$
    \expc{}{|z|^p} \leq \sqrt{2\pi\lambda^2}\cdot p\cdot\lambda^p \frac{2^{p/2}\Gamma(\frac{p+1}{2})}{\sqrt{\pi}}
$$
where $\Gamma(\cdot)$ is the Gamma function. Specifically, $\expc{}{|z|}\leq \sqrt{2\pi}\lambda$, $\expc{}{|z|^2}\leq 4\lambda^2$, and $\expc{}{|z|^4}\leq 24\lambda^4$.
\end{lemma}
\begin{proof}
Let $n\sim \cN(0,\lambda^2)$. Then, it holds that
\begin{align*}
    \expc{}{\|z\|^p} & \overset{(a)}{=} \int_0^\infty p u^{p-1} \prob{\|z\| \geq u} du \\
    & \overset{(b)}{\leq} \int_0^\infty p |u|^{p-1} 2e^{-\frac{u^2}{2\lambda^2}} du \\
    & = \sqrt{2\pi\lambda^2}\cdot p\cdot \int_{-\infty}^\infty |u|^{p-1} \ofrac{\sqrt{2\pi\lambda^2}}e^{-\frac{u^2}{2\lambda^2}} du \\
    & =\sqrt{2\pi\lambda^2}\cdot p\cdot \expc{}{|n|^{p-1}} \\
    & \overset{(c)}{=}\sqrt{2\pi\lambda^2}\cdot p\cdot \lambda^p \frac{2^{p/2}\Gamma(\frac{p+1}{2})}{\sqrt{pi}}~,
\end{align*}
where $(a)$ follows from the tail representation of the absolute moments $p \in (0,\infty)$ of a non-negative random variable $z$ \citep[Exercise 1.2.3]{vershynin2018highdim}, $(b)$ follows from the sub-Guassian assumption, $(c)$ follows from the known formula of the central absolute moments of the Gaussian distribution. 
\end{proof}

\begin{lemma} \label{lem:modifed feature covariance matrix minimal eigenvalue_apnd}
Let $\cG$ be an event such that $\prob{\cG}\leq\delta$ for some $\delta\in(0,\ofrac{2})$. Then, 
$$
    \tilde\Sigma:=\cexpc{}{\left(x - \cexpc{}{x}{\cG^c} \right)\left(x - \cexpc{}{x}{\cG^c} \right)^T}{\cG^c}\succeq \left(\rho - 105\kappa^2\sqrt{\delta}\right)\cdot I~.
$$
\end{lemma}
\begin{proof} Let $\tilde\Sigma$ be the conditional covariance matrix of the features. We first relate it to the unconditional covariance matrix by the decomposition
\begin{align}\label{eq:decomposition of modified Sigma} 
    \tilde\Sigma & = \cexpc{}{\left(xx^T \right)}{\cG^c} - \cexpc{}{x}{\cG^c}\cexpc{}{x^T}{\cG^c}\nonumber\\
    & \overset{(a)}{=}\frac{\expc{}{xx^T}-\expc{}{xx^T\bI(\cG)}}{\prob{\cG^c}} - \cexpc{}{x}{\cG^c}\cexpc{}{x^T}{\cG^c}\nonumber\\ 
    & \overset{(b)}{=} \Sigma + \underbrace{ \left(\ofrac{\prob{\cG^c}}-1\right)\expc{}{xx^T}}_{:=G_1} +
    \underbrace{\expc{}{x}\expc{}{x^T} -\cexpc{}{x}{\cG^c}\cexpc{}{x^T}{\cG^c} }_{:=G_2} - 
    \underbrace{\frac{\expc{}{xx^T\cdot\bI(\cG)}}{\prob{\cG^c}}}_{:=G_3}~,
\end{align}
where $(a)$ follows from the law of total expectation$(b)$ follows from the definition of the unconditional covariance matrix of the features $\Sigma:=\expc{}{xx^T}-\expc{}{x}\expc{}{x^T}$. For the matrix $G_1$ in the last display, the assumptions $\Sigma \succeq \rho \cdot I$ and $\delta\leq\half$ imply that
\begin{equation} \label{eq:G1 bound}
    G_1 \succeq 0~.
\end{equation}
We next bound the maximal value of $\abs{v^TG_2v}$ and $\abs{v^TG_3v}$ over all unit vectors $v\in\R^d$ (with $\norm{v}=1$). For $G_2$, we further decompose to 
\begin{align*}
    G_2 &= \expc{}{x}\expc{}{x^T} -\cexpc{}{x}{\cG^c}\cexpc{}{x^T}{\cG^c} \\
    & \overset{(a)}{=} \underbrace{\expc{}{x}\left(\expc{}{x^T} -\cexpc{}{x^T}{\cG^c}\right)}_{:=G_{2,1}} + \underbrace{\left(\expc{}{x} -\cexpc{}{x}{\cG^c} \right) \cexpc{}{x^T}{\cG^c}}_{:=G_{2,2}},
\end{align*}
where $(a)$ follows by adding and subtracting the common term $\expc{}{x}\cexpc{}{x^T}{\cG^c}$. Now, for any $v\in\R^d$ with $\norm{v}=1$, it holds that 
\begin{align}\label{eq:bound on G21}
    \abs{v^TG_{2,1}v} & = \abs{v^T \expc{}{x}\left(\expc{}{x^T} -\cexpc{}{x^T}{\cG^c}\right) v}\nonumber\\
    & \overset{(a)}{\leq} \norm{\expc{}{x}} \cdot \norm{\expc{}{x} -\cexpc{}{x}{\cG^c}}\\
    & \overset{(b)}{=} \expc{}{x} \cdot 
    \frac{\norm{(\prob{\cG^c}-1)\expc{}{x} -\expc{}{x\cdot\bI(\cG)}}}{\prob{\cG^c}}\nonumber\\
    & \overset{(c)}{\leq} \expc{}{\norm{x}} \cdot 
    \frac{\norm{(\prob{\cG^c}-1)\expc{}{x} -\expc{}{x\cdot\bI(\cG)}}}{\prob{\cG^c}}\nonumber\\
    & \overset{(d)}{\leq} \expc{}{\norm{x}} \cdot 
    \frac{(1-\prob{\cG^c})\expc{}{\norm{x}} +\expc{}{\norm{x}\cdot\bI(\cG)}}{\prob{\cG^c}}\nonumber\\
    & \overset{(e)}{\leq} \expc{}{\norm{x}} \cdot 
    \frac{(1-\prob{\cG^c})\expc{}{\norm{x}} +\sqrt{\expc{}{\norm{x}^2}\cdot\prob{\cG}}}{\prob{\cG^c}}\nonumber\\
    & \overset{(f)}{\leq} 18\kappa^2 \cdot \sqrt{\delta}~,\nonumber
\end{align}
where $(a)$ follows from Cauchy-Schwarz's inequality, $(b)$ follows from the law of total expectation, $(c)$ follows from Jensen's inequality, $(d)$ follows from the triangle inequality and Jensen's inequality, $(e)$ follows from Cauchy-Schwarz's inequality, $(f)$ follows from the assumptions that $x$ is sub-Gaussian with variance parameter $\kappa^2$ and \lemref{moments of subgaussian}, with the assumption $\prob{G}=\delta\leq \ofrac{2}$ (and as $\delta<\sqrt{\delta}$). 

For $G_{2,2}$ we use a similar bounding method, except that now the bound on $\norm{\expc{}{x}}$ is replaced by a bound on $\norm{\cexpc{}{x}{G^c}}$ (in \eqnref{bound on G21}). This conditional expectation can be bounded as follows:
\begin{align*}
    \norm{\cexpc{}{x}{G^c}} &\overset{(a)}{=} 
    \frac{\norm{\expc{}{x}+ \expc{}{x\cdot\bI(\cG)}}}{\prob{\cG^c}} \\
    & \overset{(b)}{\leq} \frac{\expc{}{\norm{x}} + \sqrt{\expc{}{\norm{x}^2}\cdot \prob{\cG}}}{\prob{\cG^c}} \\
    & \overset{(c)}{\leq} \sqrt{8\pi}\kappa + 4\kappa\sqrt{\delta}~,
\end{align*}
where $(a)$ follows from the law of total expectation, $(b)$ follows by similar steps in the analysis of $G_{2,1}$, using the triangle, Jensen's and Cauchy-Schwarz's inequalities, and $(c)$ follows from \lemref{moments of subgaussian} and the assumptions. With this bound, as in the analysis of $G_{2,1}$, it holds for any $v\in\R^d$ with $\norm{v}=1$ that 
$$
    \abs{v^TG_{2,2}v} \leq 80\kappa^2\sqrt{\delta}~,
$$
using $\delta\sqrt{\delta}\leq \delta\leq \sqrt{\delta}$. 

From the bounds on $\abs{v^TG_{2,1}v}$ and $\abs{v^TG_{2,2}v}$ we deduce that  
\begin{equation} \label{eq:G2 bound}
    \abs{v^TG_2v} = \abs{v^T(G_{2,1}+G_{2,2})v} \leq \abs{v^TG_{2,1}v} + \abs{v^TG_{2,2}v} \leq 98\kappa^2\sqrt{\delta}~.
\end{equation}
For $G_3$ it holds for any $v\in\R^d$ with $\norm{v}=1$ that 
\begin{align}\label{eq:G3 bound}
    \abs{v^TG_3v} & = \frac{\abs{\expc{}{v^Txx^Tv\cdot\bI(\cG)}}}{\prob{\cG^c}}\nonumber\\
    & \overset{(a)}{\leq} \frac{\expc{}{\norm{x}^2\cdot\bI(\cG)}}{\prob{\cG^c}}\nonumber\\
    & \overset{(b)}{\leq} \frac{\sqrt{\expc{}{\norm{x}^4}\cdot\prob{\cG}}}{\prob{\cG^c}}\nonumber\\
    & \overset{(c)}{\leq} 7\kappa^2\sqrt{\delta}~,
\end{align}
where $(a)$ follows from Cauchy-Schwarz's inequality in $\R^d$, $(b)$ follows from Cauchy-Schwarz's inequality in $L_2$, and $(c)$ follows from \lemref{moments of subgaussian} and the assumptions. 
Using the decomposition of $\tilde\Sigma$ in \eqnref{decomposition of modified Sigma} and the bounds on $G_1,G_2,G_3$ in \eqnref{G1 bound}, \eqnref{G2 bound} and \eqnref{G3 bound}, respectively, it holds for any $v\in\R^d$ with $\norm{v}=1$ that
$$
v^T\Sigma v \geq \rho + 0 - 98\kappa^2\sqrt{\delta} - 7\kappa^2\sqrt{\delta}~,$$
which directly implies to the stated claim.
\end{proof}
\subsection{Proof for \thmref{SGD for sub-Guassian}}\label{thm:SGD for sub-Guassian_apnd}
\begin{theorem}
    Given \asmpref{x_subG} and \asmpref{epsilon_subG}. Define
    $$
        R := 2 \sqrt{8 \kappa^2\log\left(\left(\frac{21 \kappa}{\sqrt{\rho}} + 8 \right) \cdot T\right)} \cdot D + \sqrt{8 \sigma^2\log\left(\left(\frac{21 \sigma}{\sqrt{\rho}} + 8 \right) \cdot T\right)}~.
    $$
    
    Then, 
    $$
        \expc{}{\norm{\bw - \sw}^2} \leq \left(\frac{D^2}{14} + \frac{288 R^2 \cdot (2\log T + 1)}{((1 - \alpha) \rho)^2}\right) \cdot \ofrac{T}~.
    $$
\end{theorem}
\begin{proof}
Let a time $T$ be given, and consider the events
$$
    \cF_x(u):= \left\{\bigcup_{i\in[{1,2,\ldots,T}]} \norm{x_i}>u\right\}~,
$$ 
and
$$
    \cF_{\epsilon}(u) := \left\{\bigcup_{i\in[{1,2,\ldots,T}]} |\epsilon_i|>u\right\}~.
$$
Further let $u_1:=\sqrt{2\kappa^2\log(\frac{4T}{\delta})}$ and $u_2:=\sqrt{2\sigma^2\log(\frac{4T}{\delta})}$ and set 
$$
    \cF:= \cF_{\epsilon}(u_1)\cup \cF_{x}(u_2).
$$ 
By a union bound over $i\in[1,2,\ldots,T]$ and computing probabilities over $\epsilon$ and $x$, the sub-Gaussian assumptions implies that $\prob{\cF}\leq \delta$. We choose $\delta=\min\left\{ \half, \frac{\rho^{2}}{210^{2}\kappa^{4}}\cdot\ofrac{T^{2}}\right\}$. Note, that with this choice of $\delta$, and by identifying $\cG=\cF$, \lemref{modifed feature covariance matrix minimal eigenvalue_apnd} implies that $\tilde\Sigma\succeq \frac{\rho}{2}\cdot I$ for all $T\geq1$.

We next evaluate the error of the SGD algorithm by considering two events -- the event $\cF^c$ in which both $\norm{x_i}$ and $\epsilon_i$ are bounded for all $i\in\{1,2,\dots,T\}$, and the event $\cF$, which has a vanishing probability $\delta=O(T^{-2})$. Specifically, by the law of total expectation
\begin{equation} \label{eq:low of total expectation for sub-Gaussian}
    \expc{}{\norm{\bw - \sw}^2} = \prob{\cF} \cdot \cexpc{}{\norm{\bw - \sw}^2}{\cF} + \prob{\cF^c} \cdot \cexpc{}{\norm{\bw - \sw}^2}{\cF^c}~,
\end{equation}
where $\cF^c$ is the complement of the event $\cF$. The first term in \eqnref{low of total expectation for sub-Gaussian} is upper bounded as follows:
\begin{align} \label{eq:MSE bound sub-Gaussian case unbounded event}
    \prob{\cF}\cexpc{}{\norm{\bw - \sw}^2}{\cF} &= \expc{}{\norm{\bw - \sw}^2 \cdot \bI(\cF)}\nonumber\\
    &\overset{(a)}{\leq} \sqrt{\expc{}{\norm{\bw - \sw}^4} \cdot \expc{}{\bI(\cF)}}\nonumber\\
    & \overset{(b)}{\leq} 4 D^2 \cdot \sqrt{\prob{\cF)}}\nonumber\\
    & \overset{(c)}{\leq} \frac{4 D^2 \rho}{210\kappa^{2}} \cdot \ofrac{T}~,\\
    & \overset{(d)}{\leq} \frac{D^2 }{14} \cdot \ofrac{T}~,
\end{align}
where $(a)$ follows from Cauchy-Schwarz's inequality, $(b)$ follows from \asmpref{w}, and the fact that $\bw \in \cW$, $(c)$ follows since $\prob{\cF} \leq \delta$ and the choice of $\delta$, and $(d)$ follows since
$$
    \rho \leq \max_{v\in\R^d\colon \norm{v}\leq 1} \expc{}{\inn{v,x-\expc{}{x}}^2} \leq
    \expc{}{\norm{x-\expc{}{x}}^2} \leq 
    \expc{}{\norm{x}^2} \leq
    4\kappa^2~.
$$
For the second term in \eqnref{low of total expectation for sub-Gaussian}, we note that conditioned on $\cF^c$ the noise and the feature vectors are bounded, that is $\|x_i\|\leq u_1$ and $\|\epsilon_i\|\leq u_2$ for all $i\in[1,2\ldots,T]$. This model is similar to the one discussed in previous sections, in particular to \mdlref{uknown_bounded} in \sectref{unknown}, where the expectation is unknown, with two differences. First, as said, by the choice of $\delta$ the conditional covariance matrix of the features has minimal eigenvalue of $\frac{\rho}{2}$, instead of $\rho$ for the unconditional covariance matrix. The second difference is that $\cexpc{}{\epsilon}{\cF^c}$ may not equal zero. However, it can be easily verified that the result of \sectref{known_apnd} holds, since the noise related terms in the second derivative of the function $F$ therein vanish.

We will follow the same steps as in \sectref{unknown}: computing $\mu$ with $T$ samples conditioned on $\cF^c$ and feed them to \algoref{unknown} with different input, that is because the radius parameter $R$, is different and will depend on $u_1$ and $u_2$.

The derivation of the new radius parameter $R$ is similar to derivation made in \sectref{huber}, that is bounding the norm of the features \& noise by $u_1$ and $u_2$ respectively. By our choice of $\delta=\min\left\{ \half, \frac{\rho^{2}}{210^{2}\kappa^{4}}\cdot\ofrac{T^{2}}\right\}$ we may further bound 
$$
    u_1 = \sqrt{2\kappa^2\log\left(4 \cdot \max\left\{ 2 T, \frac{210^{2}\kappa^{4}T^{3}}{\rho^{2}} \right\}\right)} \leq \sqrt{8 \kappa^2\log\left(\left(\frac{21 \kappa}{\sqrt{\rho}} + 8 \right) \cdot T\right)}~,
$$
and similarly
$$
    u_2 \leq \sqrt{8 \sigma^2\log\left(\left(\frac{21 \sigma}{\sqrt{\rho}} + 8 \right) \cdot T\right)}~.
$$
Then, by taking
$$
    R := 2 \sqrt{8 \kappa^2\log\left(\left(\frac{21 \kappa}{\sqrt{\rho}} + 8 \right) \cdot T\right)} \cdot D + \sqrt{8 \sigma^2\log\left(\left(\frac{21 \sigma}{\sqrt{\rho}} + 8 \right) \cdot T\right)}~,
$$
it holds by \thmref{unknown_main} that
\begin{align}\label{eq:MSE bound sub-Gaussian case bounded event}
    \prob{\cF^c} \cexpc{}{\norm{\bw - \sw}^2}{\cF^c} &\leq \frac{72 R^2 \cdot (2\log T + 1)}{((1 - \alpha) \frac{\rho}{2})^2 T}\nonumber\\
    &= \frac{288 R^2 \cdot (2\log T + 1)}{((1 - \alpha) \rho)^2 T}~,
\end{align}
where we use the fact that $0 \geq \prob{\cF^c} \leq 1$.

Note that $\sw$ is deterministic and doesn't change when conditioned of $\cF^c$, then by combining the bounds in \eqnref{MSE bound sub-Gaussian case unbounded event}
and \eqnref{MSE bound sub-Gaussian case bounded event} into \eqnref{low of total expectation for sub-Gaussian} completes the proof. 
\end{proof}

\section{Discussing Our Results with respect to the Spreadness Assumption}
The paper of \citet{pmlr-v139-d-orsi21a} assumes that a condition called spreadness applies the to empirical covariance matrix of the features. Moreover, in the context of their work they show a hardness result: i.e., that the spreadness assumption is necessary in order to obtain meaningful guarantees for any algorithm. In contrast, our work shows that one can obtain meaningful guarantees without making this assumption. The goal of the section is to show that the hardness result of \citet{pmlr-v139-d-orsi21a} does not apply under our assumptions of bounded optimal solution and bounded features, which explains how these two results can hold simultaneously without contradiction.

We will not describe the spreadness since it is unnecessary for establishing our point. Instead, we will describe the hardness result of \citet{pmlr-v139-d-orsi21a} and show that it is not relevant in our case.
\paragraph{The Example of \cite{pmlr-v139-d-orsi21a}:}
The hardness result of \citet{pmlr-v139-d-orsi21a} (Appendix A.2.1 therein) is due to the following example:
Assume an oblivious contamination as in \mdlref{simple_model} such that $\eps = 0$ w.p.$~1$, the features are $1$ dimensional i.e. $x\in\R$, and the contaminations are distributed as follows\footnote{We stick to the notations in our paper where $\alpha$ is the fraction of contaminated samples. Conversely, \citet{pmlr-v139-d-orsi21a} denote the fraction of contaminated samples by $1-\alpha$. },
$$
    \begin{cases}
        b_i\sim \mathcal{N}(0,\sigma^2) &, \text{with probability } \alpha\\
        b_i=0 &, \text{otherwise.}
    \end{cases}~, \forall i\in[1,\ldots, T]~,
    $$
here the fraction of contaminations is $\alpha$, and the adversary injects Gaussian noise with variance $\sigma^2$; note that the adversary may choose $\sigma$ to be arbitrarily large.\\
It is also assumed that the number of non-zero features in the training set is equal to $\ofrac{C(1-\alpha)}$~\footnote{In the example appearing in their appendix A.2.1. there is a typo where they mistakenly write that number of non-zero features is $\frac{T}{C(1-\alpha)}$, but it should actually be $\ofrac{C(1-\alpha)}$. We validated this with the authors of  \citet{pmlr-v139-d-orsi21a}.}, where $C>0$ is a large enough constant that does not depend on $\alpha, T$.
With these assumptions, their hardness proof is based on the following statement (Lemma A.5 therein):
\begin{lemma}[{\citet[Lemma A.5]{pmlr-v139-d-orsi21a}}]
Consider the robust linear regression task with the above assumptions. Then for any estimate $\hat{w}$ which is based on the contaminated data, there exists an optimal solution $w^*$ (of the problem with non-corrupted data) such that the following holds,
$$
    \ofrac{T}\sum_{i=1}^T \norm{\inn{x_i,\hat{w} - w^*}}^2 = \Omega\left( \frac{\sigma^2}{T}\right)~.
$$
\end{lemma}
Now since the adversary can choose $\sigma$ to be arbitrarily large this means that one cannot obtain meaningful guarantees in this case. Nevertheless, as we show below, the above counter example is meaningless under our assumptions.

\begin{lemma}\label{lem:NoSp}
    Consider the robust linear regression task with the above assumptions. 
    Also assume that there exists an optimal solution $w^*$ such that $\|w^*\|\leq D$ and that the features are bounded, i.e., $\|x\|\leq1$ w.p.$~1$. Then there exists a trivial estimator $\hat{w}$ such that the following holds,
    \begin{align}\label{eqn:DCounter}
        \ofrac{T}\sum_{i=1}^T \norm{\inn{x_i,\hat{w} - w^*}}^2 = O\left( \frac{D^2}{(1-\alpha)T}\right)~.
    \end{align}
\end{lemma}
Importantly, the above lemma implies that for this example there does exist an estimator that achieves a vanishing prediction error that does not depend on $\sigma$. 
Next we prove the above lemma,
\begin{proof}[Proof of \lemref{NoSp}]
Consider the following trivial predictor $\hat{w} =0$. In this case, since the features are bounded by $1$, and only $\ofrac{C(1-\alpha)}$ of them are non-zero we obtain,
\begin{align*}
    \ofrac{T}\sum_{i=1}^T \norm{\inn{x_i,\hat{w} - w^*}}^2
    &=
    \ofrac{T}\sum_{i=1}^T \norm{\inn{x_i,w^*}}^2 \\
    &
    \leq
    \ofrac{T}\sum_{i=1}^T \|x_i\|^2 \cdot \|\sw\|^2 \\
    &\leq
    \frac{D}{T}\sum_{i=1}^T \|x_i\|^2 \\
    &\leq
    \frac{CD}{(1-\alpha)T}~,
\end{align*}
here the first line uses $\hat{w} =0$, the second line uses Cauchy-Schwarz's inequality, later we use the fact that
$\|w^*\|\leq D$, and the last line uses $\|x_i\|\leq 1$, and the fact that only $1/C(1-\alpha)$ of them are non zero.
This concludes the proof.
\end{proof}
To conclude, under our assumptions, the hardness result of \cite{pmlr-v139-d-orsi21a} is irrelevant, and so fast rates of convergence are possible.

\end{document}